\documentclass[twoside]{article}
\usepackage{aistats2019}
\usepackage[round]{natbib}

\input{definition.tex}

\begin{document}

%

%

\twocolumn[

\aistatstitle{
Adaptive Minimax Regret against Smooth Logarithmic Losses
over High-Dimensional $\ell_1$-Balls via Envelope Complexity
}


\aistatsauthor{
Kohei Miyaguchi$^1$
\and
Kenji Yamanishi$^1$
}

\aistatsaddress{
$^{12}$The University of Tokyo, 7-3-1 Hongo, Tokyo, Japan\\
$^1$kohei\_miyaguchi@mist.i.u-tokyo.ac.jp\\
$^2$yamanishi@mist.i.u-tokyo.ac.jp
}]

\begin{abstract}
We develop a new theoretical framework, the \emph{envelope complexity},
to analyze the minimax regret with logarithmic loss functions and
derive a Bayesian predictor that adaptively achieves the minimax regret over high-dimensional $\ell_1$-balls within a factor of two.
The prior is newly derived for achieving the minimax regret
and called the \emph{spike-and-tails~(ST) prior} as it looks like.
The resulting regret bound is so simple that it is completely determined with
the smoothness of the loss function and the radius of the balls
except with logarithmic factors,
and it has a generalized form of existing regret/risk bounds.
In the preliminary experiment, we confirm that
the ST prior outperforms the conventional minimax-regret prior under non-high-dimensional asymptotics.
\end{abstract}

\newcommand{\envmmr}{\mathrm{REG}^{\rm env}}
\newcommand{\regstoc}{\mathrm{REG}^{\rm stoc}}
\newcommand{\mmr}{\mathrm{REG}^{\star}}
\newcommand{\bhull}{\mathrm{Bayes}}
\newcommand{\holine}{\overline{h}}
\newcommand{\reg}{\mathrm{REG}}
\newcommand{\lreg}{\mathrm{LREG}}
\newcommand{\blreg}{\mathrm{BLREG}}
\newcommand{\bayes}{\mathrm{Bayes}}
\newcommand{\allh}{\hat{\Hcal}}
\renewcommand{\d}{\mathrm{d}}
\newcommand{\st}{{\rm ST}}
\newcommand{\Ln}{\mathop{\rm Ln}}

\section{Introduction}

As a notion of complexity of predictive models~(sets of predictors),
\emph{minimax regret} has been considered in the literature of
online learning~\citep{cesa2006prediction} and the minimum description length~(MDL) principle~\citep{rissanen1978modeling,grunwald2007minimum}.
The minimax regret of a model $\Hcal$ is given by
\begin{align}
    \reg^\star(\Hcal)&=\inf_{\hhat\in \allh} \sup_{X\in\Xcal} \cbr{f_X(\hhat) - \inf_{h\in\Hcal} f_X(h)},
    \label{eq:minimax_regret}
\end{align}
where $f_X(h)$ denotes the loss of the prediction over data $X$ made by $h$,
$\hat{\Hcal}$ denotes the feasible predictions and $\Xcal$ is the space of data.
Here, the data may consist of a sequence of datum, $X=X^n=(X_1,\ldots,X_n)$,
and the loss maybe additive, $f_X(h)=\sum_{i=1}^{n} f_{X_i}(h)$, but we keep them implicit for generality.
The minimax regret is a general complexity measure in the sense that
it is defined without any assumptions on the generation process of $X$.
For instance,
one can bound statistical risks with $\reg^\star(\Hcal)$
regardless of the distribution of data~\citep{littlestone1989line,cesa2004generalization,cesa2008improved}.
Therefore, bounding the minimax regret and constructing the corresponding predictor $\hhat$ is important to make a good and robust prediction.

We consider that $\Hcal$ is parametrized by
a real-valued vector $\theta\in\RR^d$,
$\Hcal=\mysetinline{h_\theta}{\gamma(\theta)\le B,\;\theta\in\RR^d}$,
where $\gamma(\theta)$ denotes a radius function such as norms of $\theta$.
Thus, we may consider the luckiness minimax regret~\citep{grunwald2007minimum},
\begin{align}
    \lreg^\star(\gamma)&=\inf_{\hhat\in \allh} \sup_{X\in\Xcal} \cbr{f_X(\hhat) - \inf_{\theta\in\RR^d} \sbr{f_X(\theta)+\gamma(\theta)}},
    \label{eq:luckiness_minimax_regret}
\end{align}
instead of the original minimax regret.
Here, we abuse the notation $f_X(\theta)=f_X(h_\theta)$.
There are at least three reasons for adopting this formulation.
Firstly, as we do not assume the underlying distribution of $X$,
it may be plausible to pose a soft restriction as in~\eqref{eq:luckiness_minimax_regret} rather than the hard restriction in~\eqref{eq:minimax_regret}.
Secondly, it is straightforwardly shown that the luckiness minimax regret bounds above the minimax regret. 
Thus, it is often sufficient to bound $\lreg^\star(\gamma)$ for bounding $\reg^\star(\Hcal)$.
Finally, the luckiness minimax regret is including the original minimax regret as a special case such that $\gamma(\theta)=0$ if $\theta\in\Hcal$ and 
$\gamma(\theta)=\infty$ otherwise.
Therefore, we may avoid possible computational difficulties of the minimax regret by choosing the penalty $\gamma$ carefully.

That being said, the closed-form expression of the exact (luckiness) minimax regret is
even intractable except with few special cases~(e.g., \cite{shtar1987universal,koolen2014efficient}).

However, if we focus on information-theoretic settings,
i.e., the model $\Hcal$ is a set of probabilistic distributions,
everything becomes explicit.
Now, let predictors be sub-probability distributions $P(\cdot\mid\theta)$ and
adopt the logarithmic loss function $f_X(\theta)=-\ln \frac{\d P}{\d\nu}(X|\theta)$ with respect to an appropriate base measure $\nu$ such as counting and Lebesgue measures.
Note that a number of important practical problems such as logistic regression and data compression can be handled with this framework.
With the logarithmic loss, the closed form of the luckiness minimax regret is given by~\cite{shtar1987universal,grunwald2007minimum} as
\begin{align}
    \lreg^\star(\gamma)=\ln \int e^{-m(f_X+\gamma)} \nu(\d X)\eqdef S(\gamma),
    \label{eq:shtarkov_complexity}
\end{align}
where $m$ denotes the minimum operator given by $m(f)=\inf_{\theta\in\RR^d}f(\theta)$.
We refer to the left-hand-side value as the \emph{Shtarkov complexity}.
Moreover, when all the distributions in $\Hcal$ are i.i.d.~regular distributions of $n$-sequences $X=(X_1,\ldots,X_n)$,
under some regularity conditions,
the celebrated asymptotic formula~\citep{rissanen1996fisher,grunwald2007minimum} is given by
\begin{align}
    S(\gamma)=\frac{d}2 \ln \frac{n}{2\pi} + \int \sqrt{\det I(\theta)}e^{-\gamma(\theta)}\d \theta+o(1),
    \label{eq:asymptotic_formula}
\end{align}
where $I(\theta)$ is the Fisher information matrix and $o(1)\to 0$ as $n\to \infty$.
More importantly, although the exact minimax-regret predictor achieving $S(\gamma)$ is still intractable, the asymptotic formula implies that
it is asymptotically achieved with 
the Bayesian predictor associated with the \emph{tilted Jeffreys prior} $\pi(\d\theta)\propto
\sqrt{\det I(\theta)}e^{-\gamma(\theta)}\d\theta$.

Here, our research questions are as follows:
First, {\bf (Q1)} \emph{How can we evaluate $S(\gamma)$ in modern high-dimensional contexts?}
In particular, the asymptotic formula~\eqref{eq:asymptotic_formula} does not withstand high-dimensional learning problems where $d$ increases as $n\to \infty$.
The exact evaluation of the Shtarkov complexity~\eqref{eq:shtarkov_complexity}, on the other hand, is often intractable due to the minimum operator inside the integral.
Second, {\bf (Q2)} \emph{How can we achieve the minimax regret with computationally feasible predictors?}
It is important to provide the counterpart of the tilted Jeffreys prior in order to make actual predictions.

Regarding the above questions, our contribution is summarized as follows:
\begin{itemize}
    \item We introduce the \emph{envelope complexity}, a non-asymptotic approximation of the Shtarkov complexity $S(\gamma)$ that allows us systematic computation of its upper bounds and predictors achieving these bounds.
    In particular, we show that the regret of the predictor is characterized with the smoothness. 
\item We demonstrate its usefulness by giving a Bayesian predictor that adaptively achieves the minimax regret within a factor of two over any high-dimensional smooth models under $\ell_1$-constraints $\norm{\theta}_1\le B$.
\end{itemize}



The rest of the paper is organized as follows:
In Section~\ref{sec:bmr}, we introduce the notion of Bayesian minimax regret
as an approximation of the minimax regret within the `feasible' set of predictors.
We then develop a complexity measure called \emph{envelope complexity}
in Section~\ref{sec:ec}
as a mathematical abstraction of the Bayesian minimax regret.
We also present a collection of techniques for bounding the envelope complexity to the Shtarkov complexity.
In Section~\ref{sec:st_prior}, we utilize the envelope complexity to construct a near-minimax Bayesian predictor under $\ell_1$-penalization, namely the spike-and-tails~(ST) prior.
We also show that it achieves the minimax rate over $\Hcal=\mysetinline{\theta\in\RR^d}{\norm{\theta}_1\le B}$ under high-dimensional asymptotics.
In Section~\ref{sec:visual}, we demonstrate numerical experiments
to visualize our theoretical results.
The discussion on these results in comparison to
the existing studies are given in Section~\ref{sec:discussion}.
Finally, we conclude the paper in Section~\ref{sec:conclusion}.

\section{Bayesian Minimax Regret}
\label{sec:bmr}

The minimax regret with logarithmic loss is given by the Shtarkov complexity $S(\gamma)$.
The computation of the Shtarkov complexity $S(\gamma)$ is often intractable if we consider practical models such as deep neural networks.
This is because the landscapes of loss functions $f\in\Fcal$ are complex as the models are, and hence their minimums $m(f)$ and the complexity, which is an integral over the function of $m(f)$, are not tractable.
Moreover, computations of the optimal predictor $h^\star$ are still often intractable
even if $S(\gamma)$ are given.
For instance, the minimax-regret prediction for Bernoulli models over $n$ outcomes cost $O(n2^n)$ time.
Of course there exist some special cases for which closed forms of $\hhat$ are given.
However, so far they are limited to exponential families.

One cause of this issue is that we seek for the best predictor $\hhat$ among all the possible predictors $\allh$, i.e., all probability distributions.
This is too general that it maybe not possible to compute $\hhat$ nor $\reg^\star(\gamma)$.
To avoid this difficulty, we narrow the set of feasible predictors $\allh$ to the Bayesian predictors.
Let $w\in\Mcal_+(\RR^d)$ be a positive measure over $\RR^d$,
which we may refer to as \emph{pre-prior}, and
let $\holine_w$ be the Bayesian predictor associated with the prior $\pi(\d \theta)\propto e^{-\gamma(\theta)}w(\d\theta)$.
Then we have
\begin{align}
    f_X(\holine_w)=\ln \frac{w\sbr{e^{-\gamma}}}{w\sbr{e^{-f_X-\gamma}}}\eqdef f_X(w),
    \label{eq:definition_bayesian_loss}
\end{align}
where $w\sbr{\cdot}$ denotes the integral operation with respect to $w(\d\theta)$.
Now, we consider the Bayesian (luckiness) minimax regret given by
\begin{align*}
    \lreg^\bayes(\gamma)
    &\eqdef\inf_{w\in\Mcal_+(\RR^d)}\lreg(w|\gamma),
    \\\lreg(w|\gamma)
    &\eqdef\sup_{X\in\Xcal} \cbr{f_X(w)-m\rbr{f_X+\gamma}}.
\end{align*}


One advantage of considering the Bayesian minimax regret is that,
given a measure $w$, one can compute $\holine_w$ analytically or numerically utilizing techniques developed in the literature of Bayesian inference. In particular, a number of sophisticated variants of Monte Carlo Markov chain~(MCMC) methods such as stochastic gradient Langevin Dynamics~\citep{welling2011bayesian} are developed for sampling $\theta$ from complex posteriors.

Note that their does exist a case
where the Bayesian minimax regret strictly differs
from the minimax regret.
See \cite{barron2014bayesian} for example.
%
It implies that
narrowing the range of predictors to Bayesian
may worsen the achievable worst-case regret.
However, as we will show shortly, the gap between
these minimax regrets can be controlled
with model $\gamma$.

\section{Envelope Complexity}
\label{sec:ec}
We have introduced the Bayesian minimax regret $\lreg^\bayes(\gamma)$.
In this section, we present a set representation of Bayesian minimax regret,
namely the \emph{envelope complexity} $C(\gamma, \Fcal)$.
Then, we show that the Shtarkov complexity is bounded by the envelope complexity and
the envelope complexity can be easily bounded
even if the models are complex.

\subsection{Set Representation of Bayesian Minimax Regret}
The envelope complexity is a simple mathematical abstraction of Bayesian minimax regret
and gives a fundamental basis for systematic computation of upper bounds on the (Bayesian) minimax regret.
Let $\Fcal$ be a set of continuous functions $f:\RR^d\to\RR$ which is not necessarily logarithmic.
Define the Bayesian envelope of $\Fcal$ as
\begin{align*}
    \Ecal(\Fcal)\eqdef \myset{w\in\Mcal_+(\RR^d)}{\forall f\in\Fcal,w\sbr{e^{-f+m(f)}}\ge 1},
\end{align*}
and define the envelope complexity as
\begin{align*}
    C(\gamma, \Fcal)\eqdef
    \inf_{w\in\Ecal(\Fcal)} \ln w\sbr{e^{-\gamma}}.
\end{align*}
Then, the envelope complexity characterizes Bayesian minimax regret.

\begin{theorem}[Set representation]
    \label{thm:set_repr_envelope}
    Let $\Fcal=\mysetinline{f_X+\gamma}{X\in\Xcal}$.
    Then, all measures in the envelope $w\in\Ecal(\Fcal)$ satisfies that
    \begin{align*}
        \lreg(w|\gamma)\le \ln w\sbr{e^{-\gamma}}.
    \end{align*}
    Moreover, we have
    \begin{align*}
        \lreg^\bayes(\gamma)=C\rbr{\gamma,\Fcal}.
    \end{align*}
\end{theorem}
\begin{proof}
    Let $c(w)=\inf_{f\in\Fcal}w[e^{-f+m(f)}]$.
    Observe that
    \begin{align*}
        \ln \frac{w\sbr{e^{-\gamma}}}{c(w)}
      &=\sup_{f\in \Fcal}\cbr{ \ln \frac{w\sbr{e^{-\gamma}}}{w\sbr{e^{-f}}} - m(f)}
      \\&=\sup_{X\in \Xcal}
        \cbr{ \ln \frac{w\sbr{e^{-\gamma}}}{w\sbr{e^{-f_X-\gamma}}} - m(f_X+\gamma)}
        \\&\qquad (f=f_X+\gamma)
      \\&=\lreg(w|\gamma). 
        \\&\qquad (\because~\eqref{eq:definition_bayesian_loss})
    \end{align*}
    Then, since $c(w)\ge 1$ for all $w\in\Ecal(\Fcal)$,
    we have the first inequality.
    
    Note that $\wbar=w/c(w)\in \Ecal(\Fcal)$ for any $w\in\Mcal_+(\RR^d)$,
    and $\wbar\sbr{e^{-\gamma}}\le w\sbr{e^{-\gamma}}$ whenever $w\in\Ecal(\Fcal)$.
    Then we have
    \begin{align*}
        C(\gamma, \Fcal)
        &=\inf_{w\in\Mcal_+(\RR^d)} \ln \frac{w\sbr{e^{-\gamma}}}{c(w)}
      \\&=\inf_{w\in\Mcal_+(\RR^d)} \lreg(w|\gamma)
        &(\mathrm{the\ above\ equality})
      \\&=\lreg^\bayes(\gamma),
    \end{align*}
    yielding the second equality.
    This completes the proof.
\end{proof}


We have seen that the envelope complexity is equivalent to the Bayesian minimax regret.
Below, we present upper bounds of the Shtarkov complexity we put our basis on in the rest of the paper.

\begin{theorem}[Bounds on Shtarkov complexity]
    \label{thm:upper_lower_bounds_of_sc_via_ec}
    Let $\Fcal=\mysetinline{f_X+\gamma}{X\in\Xcal}$ where $f_X$ is logarithmic.
    Then, for all $w\in\Ecal(\Fcal)$,
    we have
    \begin{align*}
        S(\gamma) \le C(\gamma, \Fcal)\le  \ln w\sbr{e^{-\gamma}}.
    \end{align*}
\end{theorem}
\begin{proof}
    The first inequality follows from that
    the envelope minimax regret is no less than the minimax regret,
    as the range of infimum is shrunk from $\allh$ to the Bayes class $\cbr{\holine_w }$.
    The second inequality is seen by that the definition of the envelope complexity.
    This completes the proof.
\end{proof}

%
\subsection{Useful Lemmas for Evaluating Envelope Complexity}

Next, we show several lemmas that highlight the computational advantage of the envelope complexity.
We start to show that the envelope complexity is easily evaluated with the surrogate relation.
We say a function $g$ is \emph{surrogate} of another function $f$ if and only if $f-m(f)\le g-m(g)$,
which is denoted by $f\preceq g$.
Moreover, if there is one-to-one correspondence between $g\in\Gcal$ and $f\in\Fcal$ such that $f\preceq g$,
then we may write $\Fcal\preceq \Gcal$.

\begin{lemma}[Monotonicity]
    \label{lem:monotonicity_ec}
    Let $\Fcal\preceq\Gcal'\subset \Gcal$.
    Then we have
    \begin{align*}
        .\Ecal(\Fcal)\supset \Ecal(\Gcal)
    \end{align*}
    and therefore
    \begin{align*}
        C(\gamma, \Fcal)\le C(\gamma, \Gcal).
    \end{align*}
\end{lemma}
\begin{proof}
    Note that $e^{-f+m(f)}\ge e^{-g+m(g)}$ if $f\preceq g$,
    which means $\Ecal(\Fcal)\supset \Ecal(\Gcal')$.
    Also, as increasing the argument from $\Gcal'$ to $\Gcal$ just
    strengthen the predicate of the envelope, we have $\Ecal(\Gcal')\supset \Ecal(\Gcal)$.
    Therefore, we have
    \begin{align*}
        C(\gamma, \Fcal)
        &=
        \inf_{w\in\Ecal(\Fcal)}\ln w\sbr{e^{-\gamma}}
      \\&\le \inf_{w\in\Ecal(\Gcal')}\ln w\sbr{e^{-\gamma}}
        &\Ecal(\Fcal)\supset \Ecal(\Gcal')
      \\&\le \inf_{w\in\Ecal(\Gcal)}\ln w\sbr{e^{-\gamma}}
        &\Ecal(\Gcal')\supset \Ecal(\Gcal)
      \\&=C(\gamma,\Gcal).
    \end{align*}
\end{proof}
This is especially useful when the loss functions $\Fcal$ are complex but there exist simple surrogates $\Gcal$.
Consider any models such that the landscapes of the associated loss functions $f\in\Fcal$
are not fully understood and the evaluation of $m(f)$ is expensive.
It is impossible to check if $w$ is in the envelope, $w\in\Ecal(\Fcal)$,
and therefore Theorem~\ref{thm:upper_lower_bounds_of_sc_via_ec} cannot be used directly.
However, even in such cases, one can possibly find a surrogate class $\Gcal$ of $\Fcal$.
If the surrogate $\Gcal$ is simple enough for checking if $w\in\Ecal(\Gcal)$, it is possible to bound the envelope complexity utilizing Lemma~\ref{lem:monotonicity_ec} and Theorem~\ref{thm:upper_lower_bounds_of_sc_via_ec}.

In what follows, we consider the specific instance of the surrogate relation based on the smoothness.
A function $f:\RR^d\to \RR$ is $L$-\emph{upper smooth} if and only if, for all $\theta,\theta_0\in\RR^d$,
there exists $g\in\RR^d$ such that
\begin{align}
    f(\theta)\le f(\theta_0)+g^\top (\theta-\theta_0)+ \frac{L}{2}\norm{\theta-\theta_0}_2^2.
    \label{eq:upper_smooth}
\end{align}
Note that the upper smoothness is weaker than (Lipschitz) smoothness.
Thus, if $f$ is $L$-upper smooth and has at least one minima $\theta_0\in\arg m(f)$,
we can construct a simple quadratic surrogate of $f$, $\theta\mapsto \frac{L}{2}\norm{\theta-\theta_0}_2^2~(\succeq f)$.

Motivated by the smoothness assumption,
below we present more specific bounds for quadratic functions.
Let $\Qcal$ be the set of all quadratic functions with curvature one,
defined as $\Qcal=\mysetinline{\theta\mapsto\frac 12\norm{\theta-u}^2}{u\in\RR^d}$.
Moreover, for all sets of loss functions $\Fcal$ and penalty functions $\gamma:\RR\to \RRbar$,
we write $\Fcal_\gamma=\Fcal+\gamma=\mysetinline{f+\gamma}{f\in\Fcal}$.
Then, the envelope complexity of $\Fcal_\gamma$ is evaluated with that of $\Qcal_\gamma$.

\begin{lemma}[Bounds of smoothness]
    \label{lem:ec_bound_with_quadratic}
    Suppose that all $f\in\Fcal$ are $L$-upper smooth.
    Let $\varphi(\theta)=\sqrt{L}^{-1}\theta$ be the scaling function.
    Then we have
    \begin{align*}
        \Ecal(\Qcal_{\gamma\circ \varphi})\circ \varphi^{-1}
        \subset \Ecal(\Fcal_\gamma),
    \end{align*}
    and moreover,
    \begin{align*}
        C(\gamma, \Fcal_\gamma)
        \le
        C(\gamma\circ\varphi, \Qcal_{\gamma\circ\varphi}).
    \end{align*}
\end{lemma}
\begin{proof}
    Note that $\Fcal_\gamma\preceq (L\Qcal)_\gamma=(\Qcal\circ \varphi^{-1})_\gamma$ since $\Fcal$ is a set of $L$-upper smooth functions.
    Observe that, for all $\Fcal$,
    \begin{align*}
        \Ecal(\Fcal\circ\varphi)
        &=\myset{w}{w\sbr{e^{-f\circ\varphi-m(f\circ\varphi)}}\ge1,\;\forall f\in\Fcal }
        \\&=\myset{w}{w\circ \varphi^{-1}\sbr{e^{-f-m(f)}}\ge1,\;\forall f\in\Fcal }
        \\&=\myset{\wtil\circ \varphi}{\wtil\sbr{e^{-f-m(f)}}\ge1,\;\forall f\in\Fcal }
        \\&=\Ecal(\Fcal)\circ \varphi,
    \end{align*}
    where $w$ and $\wtil$ range over $\Mcal_+(\RR^d)$.
    Thus, by Lemma~\ref{lem:monotonicity_ec},
    we have $\Ecal(\Fcal_\gamma)\supset \Ecal((\Qcal\circ\varphi^{-1})_\gamma)=\Ecal(\Qcal_{\gamma\circ\varphi})\circ\varphi^{-1}$.
    This proves the inclusion.
    Now we also have
    \begin{align*}
        C(\gamma, \Fcal_\gamma)
        &=\inf_{w\in\Ecal(\Fcal_\gamma)}\ln w\sbr{e^{-\gamma}}
        \\&\le\inf_{w\in\Ecal(\Qcal_{\gamma\circ\varphi})\circ\varphi^{-1}}\ln w\sbr{e^{-\gamma}}
        \\&=\inf_{w\in\Ecal(\Qcal_{\gamma\circ\varphi})}\ln w\circ\varphi^{-1}\sbr{e^{-\gamma}}
        \\&=\inf_{w\in\Ecal(\Qcal_{\gamma\circ\varphi})}\ln w\sbr{e^{-\gamma\circ\varphi}}
        \\&=C(\gamma\circ \varphi, \Qcal_{\gamma\circ\varphi}),
    \end{align*}
    which yields the inequality.
    %
\end{proof}

This lemma shows that, as long as we consider the envelope complexity of of upper smooth functions $\Fcal$,
it suffices for bounding above them to evaluate the envelope complexity of penalized quadratic functions $\Qcal_\gamma$.

Further, according to the lemma below, we can restrict ourselves to one-dimensional parametric models w.l.o.g.~if the penalty functions $\gamma$ is separable.
Here, $\gamma$ is said to be separable if and only if
it can be written in the form of $\gamma(\theta)=\sum_{j=1}^d \gamma_j(\theta_j)$.

\begin{lemma}[Separability]
    \label{lem:ec_separability}
    Suppose that $\gamma$ is separable.
    Then, the envelope complexity of $\Qcal_\gamma$ is bounded by a separable function, i.e.,
    \begin{align*}
        C(\gamma, \Qcal_\gamma)&\le \sum_{j=1}^d C(\gamma_j, \Qcal^{1}_{\gamma_j}),
    \end{align*}
    where $\Qcal^{1}$ is the set of normalized one-dimensional quadratic functions with curvature one,
    $\Qcal^1=\mysetinline{x(\in\RR)\mapsto \frac12 (x-u)^2}{u\in\RR}$.
\end{lemma}
\begin{proof}
    Note that all $f\in\Qcal_\gamma$ is separable, i.e., $f(\theta)=\sum_{j=1}^d f_j(\theta_j)$ where $f_j\in \Qcal^1_{\gamma_j}$ and $\gamma(\theta)=\sum_{j=1}^d \gamma_j(\theta_j)$.
    Let $\Ecal^d=\Ecal(\Qcal^1_{\gamma_1})\otimes\cdots \otimes\Ecal(\Qcal^1_{\gamma_d})$.
    Then we have
    \begin{align*}
    C(\gamma, \Qcal_\gamma)
    &=\inf_{w\in\Ecal(\Qcal_\gamma)} \ln w[e^{-\gamma}]
    \\
    &\le \inf_{w\in\Ecal^d} \ln w[e^{-\gamma}]
    & \Ecal^d\subset \Ecal(\Qcal_\gamma)
    \\
    &= \sum_{j=1}^d\inf_{w_j\in\Ecal(\Qcal^1_{\gamma_j})} \ln w_j[e^{-\gamma_j}]
    \\
    &= \sum_{j=1}^d C(\gamma_j, \Qcal^1_{\gamma_j})
    .
    \end{align*}
\end{proof}

\paragraph{Summary}
We have defined the Bayesian envelope and envelope complexity.
The envelope complexity $C(\gamma, \Fcal)$ is equal to the Bayesian minimax regret if $\Fcal$ is the set of penalized logarithmic loss functions.
Any measures $w$ in the Bayesian envelope $\Ecal(\Fcal)$ can be utilized for bounding the Shtarkov complexity through the envelope complexity.
Most importantly, the envelope complexity satisfies some useful properties such as
monotonicity, parametrization invariance and separability.
Specifically, the monotonicity differentiate the envelope complexity from the Shtarkov complexity.

\section{The Spike-and-Tails Prior for High-Dimensional Prediction}
\label{sec:st_prior}
We leverage the envelope complexity to give a Bayesian predictor closely achieving $\lreg^\star(\gamma)$ where $\gamma(\theta)=\lambda\norm{\theta}_1$, namely, the spike-and-tails prior.
Moreover, the predictor is shown to be also approximately minimax without luckiness where $e^n\ge d/\sqrt{n}\to\infty$.

\subsection{Envelope Complexity for $\ell_1$-Penalties}
Let $\gamma$ be the weighted $\ell_1$-norm given by
\begin{align}
    \gamma(\theta)=\lambda \norm{\theta}_1,
    \label{eq:weighted_l1_penalty}
\end{align}
where $\lambda>0$.
Let $\pi_\lambda$ be the spike-and-tails~(ST) prior over $\RR^d$ given by
\begin{align}
    &\pi^\st_\lambda(\d \theta) \propto e^{-\lambda\norm{\theta}_1}
    \prod_{j=1}^d w_\lambda^\st(\d\theta_j),
    \\&w_\lambda^\st(\d x) =
        \delta_0(\d x)+\frac{e^{\lambda^2/2}}{\lambda^2e}
        \1\cbr{\abs{x}\ge \lambda}\d x,
    \label{eq:prior_l1}
\end{align}
where $\delta_{t}$ denotes Kronecker's delta measure at $t$.
We call it the spike-and-tails prior
because it consists of a delta measure~(spike) and two exponential distributions~(tails)
as shown in Figure~\ref{fig:st_prior}.

Then, envelope complexities for quadratic loss functions can be bounded as follows.

\begin{figure}[htbp]
    \centering
    \includegraphics[width=3in]{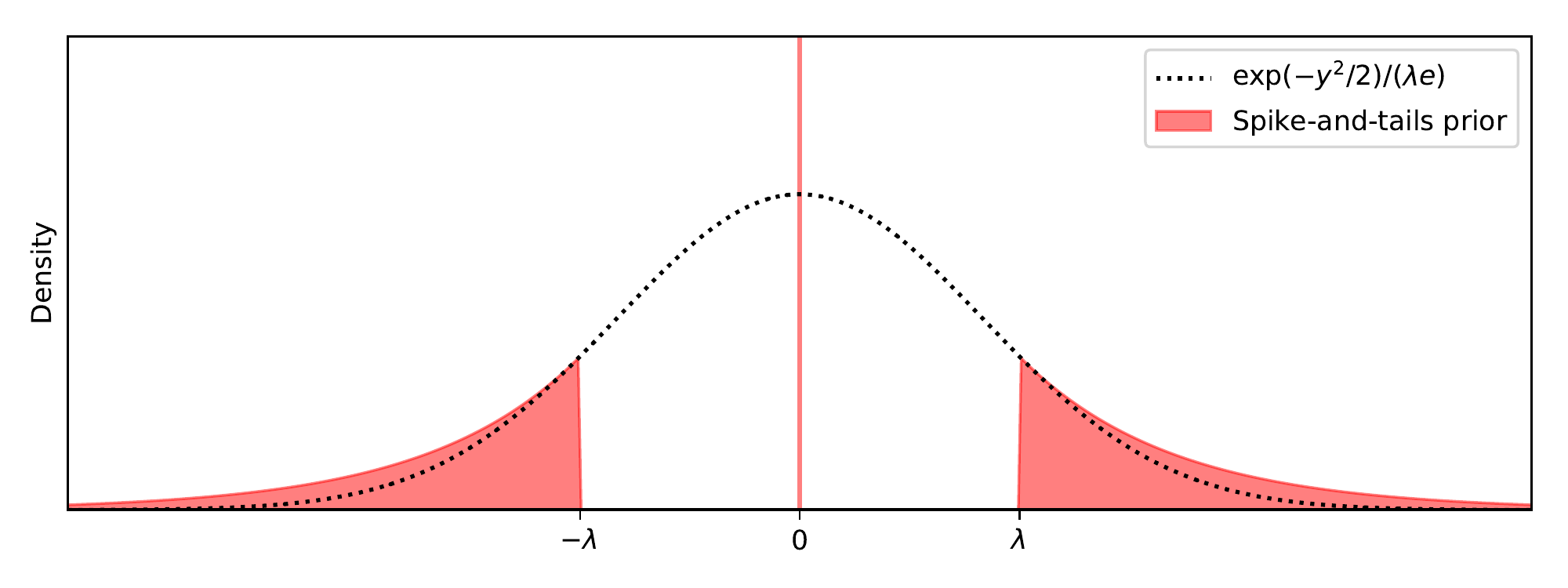}
    \caption{Density of the spike-and-tails~(ST) prior}
    \label{fig:st_prior}
\end{figure}

\begin{lemma}[Sharp bound on envelope complexity]
    \label{lem:ec_quad_l1}
    Take $\gamma$ as given by~\eqref{eq:weighted_l1_penalty}.
    Then, we have $w_{\lambda}^\st\in\Ecal(\Qcal_\gamma)$ and
    \begin{align*}
        d\ln \rbr{1+\frac{e^{-\lambda^2/2}}{\lambda^3(c+o(1))}}
        &\le 
        C(\gamma, \Qcal_\gamma)\le \ln w_\lambda^\st\sbr{e^{-\gamma}}
        \\&=
        d \ln \rbr{1+\frac{2e^{-\lambda^2/2}}{\lambda^2e}}
    \end{align*}
    for some constant $c$,
    where $o(1)\to 0$ as $\lambda\to \infty$.
\end{lemma}
\begin{proof}
    Consider the logarithmic loss functions of the $d$-dimensional standard normal location model,
    given by $f_X(\theta)=\frac12\norm{X-\theta}_2^2+\frac{d}2\ln 2\pi,~X\in\Xcal=\RR^d$ and
    let $\Fcal=\mysetinline{f_X}{X\in\RR^d}$.
    Note that $\Fcal\preceq \Qcal$.
    Then, the lower bound follows from Lemma~\ref{lem:lower_bound_sc_std_norm_loc} in Section~\ref{sec:lower_bound_sc_std_norm_loc}.
    with $S(\gamma)\le C(\gamma, \Fcal_\gamma)\le C(\gamma, \Qcal_\gamma)$.
    
    Note that $\gamma$ is separable and by Lemma~\ref{lem:ec_bound_with_quadratic}, we restrict ourselves to the case of $d=1$.
    Let $c$ and $t$ be positive real numbers.
    Let $w=\delta+cU$ be a measure over the real line,
    where $\delta$ denotes the delta measure and
    $U$ denotes the Lebesgue measures restricted to $[-\lambda, \lambda]^c=\RR\setminus [-t, t]$.
    That is, we have $w(E)=\1_{0\in E}+c\abs{E\setminus[-t, t]}$
    for measurable sets $E\subset \RR$.
    Then we have
    \begin{align}
        \ln w\sbr{e^{-\gamma}}=\ln \rbr{1 + \frac{2c}{\lambda}e^{-t\lambda}}.
        \label{eq:tmp_1}
    \end{align}
    We want to minimize \eqref{eq:tmp_1} with respect to $w\in\Ecal(\Qcal_\gamma)$.
    Let $f_u(\theta)=\frac12 \rbr{\theta-u}^2+\lambda \abs{\theta}$.
    Then we have $m(f_u)=\frac12 u^2$ if $\abs{u}\le \lambda$,
    and $m(f_u)=\lambda\abs{u}-\frac12 \lambda^2$ otherwise.
    It suffices for $c$ and $t$ to have $w\sbr{e^{-f_u}}\ge e^{-m(f_u)}$ for all $u\in \RR$.
    Here, we only care about the case of $u\ge \lambda$
    since it is symmetric with respect to $u$ and
    trivially we have $w\sbr{e^{-f_u}}\ge \delta\sbr{e^{-f_u}}\ge e^{-m(f_u)}$ for all $u\in[-\lambda, \lambda]$.
    Now, for $x=u-\lambda\ge 0$, we have
    \begin{align*}
        w\sbr{e^{-f_u}}
        &= e^{-\frac12u^2}
        + ce^{-t\lambda}
        \rbr{\int_{-\infty}^{-t}+\int_t^\infty} 
        e^{-\frac12 (\theta-u)^2} d\theta
        \\
        &\ge e^{-\frac{1}2u^2}
        + ce^{-t\lambda}
        \int_t^\infty
        e^{-\frac12 (\theta-u)^2} d\theta
        \\
        &=
        e^{-m(f_u)}
        \rbr{
            e^{-\frac12 x^2}+
            c\int_{t-x}^\infty
            e^{-\frac12 y^2} dy
        }.
    \end{align*}
    Let $A(x)=e^{-\frac12 x^2}+c\int_{t-x}^\infty e^{-\frac12 y^2} dy$.
    Thus a sufficient condition for $w\in\Ecal(\Qcal_\gamma)$ is that
    $A'(x)=ce^{-\frac12 (t-x)^2}-xe^{-\frac12 x^2}\ge 0$, which is satisfied with $c=\frac1{t}\exp\rbr{\frac12 t^2-1}$.
    Finally, evaluating \eqref{eq:tmp_1} at $t=\lambda$ yields the ST pre-prior $w=w_\lambda^\st$.
    Therefore, we have $w_\lambda^\st\in\Ecal(\Qcal_\gamma)$
    and the upper bound is shown.
    The equality is a result of straightforward calculation of $\ln w\sbr{e^{-\gamma}}$.
%
\end{proof}
According to Lemma~\ref{lem:ec_quad_l1},
the ST prior bounds the envelope complexity in a quadratic rate as $\lambda\to \infty$.
The exponent, $-\frac12\lambda^2/2$, is optimally sharp since the lower  bound $C(\gamma,\Qcal_\gamma) =\Omega(d\exp\sbr{-\frac{1}2 \lambda^2 }/\lambda^3)$ has the same exponent.

This gives an upper bound on the envelope complexity for general smooth loss functions.
Let $\pi_{\lambda,L}^\st$ and $w_{\lambda,L}^\st$ be the scale-corrected ST (pre)~prior given by
\begin{align*}
    \pi_{\lambda,L}^\st(\d\theta)=\pi_{\lambda/\sqrt{L}}^\st(\sqrt{L}\d\theta),&&
    w_{\lambda,L}^\st(\d\theta)=w_{\lambda/\sqrt{L}}^\st(\sqrt{L}\d\theta).
\end{align*}
The following is a direct corollary of Lemma~\ref{lem:ec_bound_with_quadratic}, \ref{lem:ec_separability}, \ref{lem:ec_quad_l1} and \ref{lem:monotonicity_ec}.
\begin{corollary}
    \label{cor:ec_ell1}
    If all $f\in\Fcal$ is $L$-upper smooth with respect to $\theta$,
    and if $\gamma$ is given by \eqref{eq:weighted_l1_penalty},
    then $w_{\lambda,L}^\st\in\Ecal(\Fcal_\gamma)$ and therefore
    \begin{align*}
        C(\gamma, \Fcal_\gamma)
        \le \ln w_{\lambda,L}^\st\sbr{e^{-\gamma}}=
        d \ln \rbr{1+\frac{2L}{e\lambda^2}e^{-\frac{1}{2L}\lambda^2}}.
    \end{align*}
\end{corollary}

%
\subsection{Regret Bound with the ST Prior}

Now, we utilize Corollary~\ref{cor:ec_ell1} for bounding actual prediction performance of the ST prior.
Here we consider the scenario of the online-learning under $\ell_1$-constraint.

\paragraph{Setup}
Let $X^n=(X_1,\ldots,X_n)\in\Xcal^n$ be a sequence of outcomes.
Let $f_X$ be a logarithmic loss function such that
$\int e^{-f_{X}(\theta)}d\nu(X)\le 1$.
Then, the conditional Bayesian pre-posterior with respect to $w\in\Mcal_+(\RR^d)$ given $X^t~(0\le t\le n)$ is given by
\begin{align*}
    w(\d\theta|X^t)=w(\d\theta)\prod_{i=1}^t \exp\cbr{-f_{X_i}(\theta)}.
\end{align*}
The online regret of the predictor is defined as
\begin{align}
    &\reg_n(w|\Hcal)\eqdef \nonumber 
    \\&
    \sup_{X^n\in\Xcal^n,\theta^*\in\Hcal}\sum_{t=1}^n \cbr{
        f_{X_t}(w(\cdot|{X^{t-1}})) - f_{X_t}(\theta^*)
    }.
    \label{eq:online_regret}
\end{align}
Now, we can bound the online regret of the ST prior as follows.

\begin{theorem}[Adaptive minimaxity over $\ell_1$-balls]
    \label{thm:regret_bound}
    Suppose that $f_{X_i}$ are $L$-upper smooth and logarithmic.
    Let $\Hcal_B=\mysetinline{\theta\in\RR^d}{\norm{\theta}_1\le B}$.
    Take $\lambda=\sqrt{2Ln\ln (d/\sqrt{Ln})}$.
    Then, with $\omega(1)= \ln (d/\sqrt{n})=o(n)$, we have
    \begin{align*}
        \reg_n(w_{\lambda,Ln}^\st|\Hcal_B)
        \le  B\sqrt{2Ln\ln \frac{d}{\sqrt{Ln}}}(1+o(1))
    \end{align*}
    for all $B>0$.
    Moreover, this is adaptive minimax rate and not improvable more than a factor of two
    even if $B$ is fixed and non-Bayesian predictors are involved.
\end{theorem}
\begin{proof}
    Let $f_{X^n}$ be the cumulative loss, $f_{X^n}=\sum_{i=1}^nf_{X_i}$, and observe that $f_{X^n}$ is $Ln$-upper smooth and logarithmic.
    Let $\Fcal=\mysetinline{f_{X^n}}{X^n\in\Xcal^n}$ and $\gamma(\theta)=\lambda \norm{\theta}_1$.
    Also, let $\gamma_0$ be the indicator penalty of the set $\Hcal_B$
    such that
    $\gamma_0(\theta)=0$ if and only if $\theta\in\Hcal_B$ and otherwise
    $\gamma_0(\theta)=\infty$.
    Then, we have $\reg_n(w|\Hcal_B)=\lreg(w|\gamma_0)$
    where $\lreg$ is taken with respect to $f_{X^n}$.
    Now, observe that
    \begin{align*}
        \lreg(w_{\lambda,Ln}^\st|\gamma_0)
        &\le\lreg(w_{\lambda,Ln}^\st|\gamma - \lambda B)
        \\&\quad (\because \gamma_0\ge \gamma - \lambda B)
        \\&\le
        \ln w_{\lambda,Ln}^\st\sbr{e^{-\gamma+\lambda B}},
        \\&\quad (\because\mathrm{Theorem~\ref{thm:set_repr_envelope}})
        \\&=
        \lambda B + \ln w_{\lambda,Ln}^\st\sbr{e^{-\gamma}},
    \end{align*}
    which, combined with Corollary~\ref{cor:ec_ell1} where $\lambda=\sqrt{2Ln\ln (d/\sqrt{Ln})}$,
    yields the asymptotic equality.
    The proof of the minimaxity is adopted from the existing analysis on the minimax \emph{risk}~(see Section~\ref{sec:proof_minimaxity} for the rigorous proof and Section~\ref{sec:discussion_minimax_risk} for detailed discussions).
\end{proof}

\section{Visual Comparison of the ST Prior and the Tilted Jeffreys Prior}
\label{sec:visual}
Now, we verify the results on the $\ell_1$-regularization obtained above.
In particular, we compare the worst-case regrets achievable with Bayesian predictors
to the minimax regret, i.e., the Shtarkov complexity.

\paragraph{Setting}
We adopted the one-dimensional quadratic loss functions with curvature one, $q\in\Qcal^1$,
and the $\ell_1$-penalty function, $\gamma(\theta)=\lambda\abs{\theta}$.
We varied the penalty weight $\lambda$ from $10^{-1}$ to $10^{1}$ and observed
how the worst-case regret of each Bayesian predictor changes.
Specifically, we employed the spike-and-tails~(ST) prior~\eqref{eq:prior_l1} and the tilted Jeffreys prior for the predictors.
Note that, in this case, the tilted Jeffreys prior is nothing more than the double exponential prior
given by $\pi^{\rm Jeff'}_\lambda(\d\theta)=\frac{\lambda}{2}e^{-\lambda\abs{\theta}}\d\theta$.

\paragraph{Results}
In Figure~\ref{fig:wreg_st_prior},
the worst-case regrets of the ST prior and the Jeffreys prior are shown
along with the minimax regret~(Optimal).
While the regret of the tilted Jeffreys prior is almost same as the optimal regret where $\lambda$ is small,
it performs poorly where $\lambda$ is large.
On the other hand, the ST prior performs robustly well in the entire range of $\lambda$.
Specifically, it converges to zero quadratically where $\lambda$ is large.
Therefore, since one must take $\lambda$ sufficiently large
if $d$ is large, it is implied that the ST prior is
a better choice than the tilted Jeffreys prior.

\begin{figure}[htbp]
    \centering
    \includegraphics[width=3in]{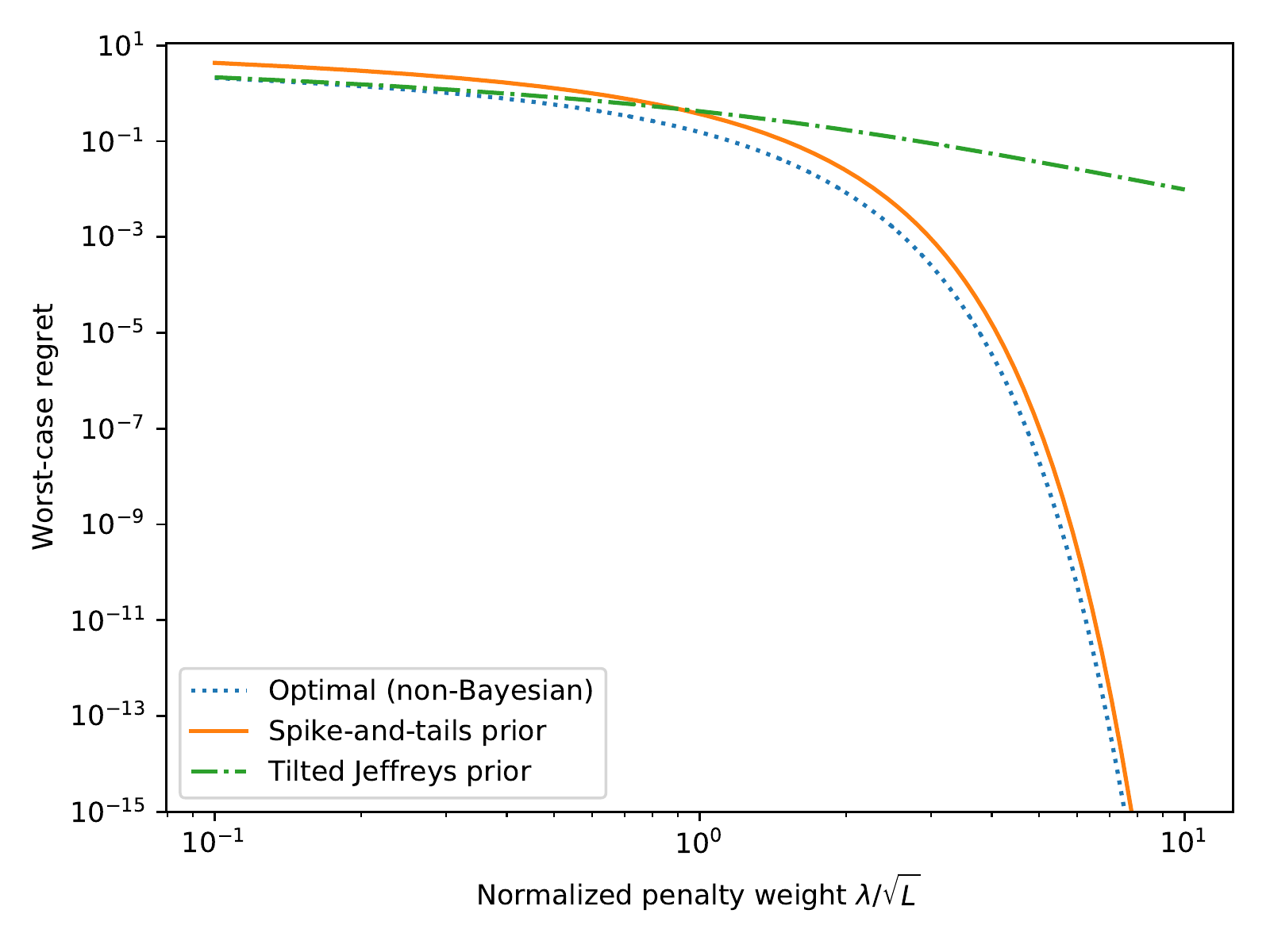}
    \caption{Worst-case regrets of the spike-and-tails~(ST) prior and the tilted Jeffreys prior}
    \label{fig:wreg_st_prior}
\end{figure}

\section{Implications and Discussions}\label{sec:discussion}

In this section, we discuss interpretations of the results and
present solutions to some technical difficulties.

\subsection{Gap between $\lreg^\star$ and $\lreg^\bayes$}

One may wonder if there exists a prior that achieves the lower bound $\lreg^\star(\gamma)$ where $\gamma(\theta)=\lambda\norm{\theta}_1,~\lambda>0$.
Unfortunately, the answer is negative.
With a similar technique of higher-order differentiations used by~\cite{hedayati2012optimality},
we can show that, if $\gamma$ is convex and not differentiable like the $\ell_1$-norm, then
the gap is nonzero, i.e., $\lreg^\star(\gamma)<\lreg^\bayes(\gamma)$.
The detailed statement and the proof is in Section~\ref{sec:gap_bayesian_lreg}.

\subsection{Infinite-dimensional Models}
If the dimensionality $d$ of the parameter space is countably infinite,
the minimax regret $\reg^\star(\Hcal_B)$ with any nonzero radius $B$ diverges.
In this case, one may apply different penalty weights to different dimensions.
For instance, taking different penalty weights for different dimensions,
e.g.,
$\gamma(\theta)=\sum_{j=1}\lambda_j\abs{\theta_j}$
for $\lambda_j=\sqrt{2L \Ln \cbrinline{j\Ln j}}$
and $\Ln x=\ln\max \cbr{e, x}$,
the separability of the envelope complexity guarantees that
$C(\gamma, \Fcal_\gamma)\le \sum_{j=1}^{\infty}\rbr{j\Ln^2 j}^{-1}< +\infty$.
Then, the corresponding countably-infinite tensor product of the one-dimensional ST prior 
$\pi^{\st}_{\cbr{\lambda_j}}(\d\theta)=\prod_{j=1}^\infty \pi^\st_{\lambda_j}(\d\theta_j)$
gives a finite regret with respect the infinite-dimensional models $\Hcal=\mysetinline{\theta\in\RR^\NN}{\gamma(\theta)\le B}$.

\subsection{Comparison to the Titled Jeffreys Priors and Others}

There have been previous studies on the minimax regret with Bayesian predictors~\citep{Takeuchi1998,takeuchi2013asymptotically,watanabe2015achievability,xie2000asymptotic}.
In these studies, the Bayesian predictor based on the Jeffreys prior~(namely Jeffreys predictor) is proved to attain minimax-regret asymptotically under some regularity conditions.
The tilted Jeffreys prior, which takes the effect of penalization $\gamma$ into consideration, is given by \cite{grunwald2007minimum} as
$
\pi_{\rm Jeff'}(\d\theta)\propto \d\theta \sqrt{\det I(\theta)}e^{-\gamma(\theta)}
$,
where $I(\theta)$ denotes the Fisher information matrix.
In the case of quadratic loss functions $\Qcal$, as the Fisher information is equal to identity, we have
$
\pi_{\rm Jeff'}(\d\theta)\propto e^{-\gamma} \d\theta 
$.
Therefore, it implies that taking the uniform pre-prior $w(\d\theta)\propto\d\theta$ is good for smooth models under
the conventional large-sample limit.
This is in very strong contrast with our result,
where completely nonuniform preprior $w_\lambda^\st$ performs better
with high-dimensional models.

\subsection{Comparison to Online Convex Optimization}
So far, we have considered the luckiness minimax regret,
which leads to the adaptive minimax regret.
Perhaps surprizingly, our minimax regret bound coincides with
the results given in the literature of online convex optimization,
where different assumptions on the loss functions and predictors are made.
Specifically, with $\lambda=\sqrt{2L\ln d}$,
the regret bound is reduced to $\sqrt{2L\ln d}+1/e$.
This coincides with the standard no-regret rates of online learning such as Hedge algorithm~\citep{freund1997decision} and high-dimensional online regression~\citep{gerchinovitz2014adaptive},
where $L$ is referred to as the number of trials $T$ and $d$ is referred to as the number of experts or dimensions $n$.
Moreover, with $\lambda=1$,
the regret bound is reduced to $O(d\ln L)$.
This is equal to the minimax-regret rate achieved under large-sample asymptotics such as in~\cite{hazan2007logarithmic,cover2011universal}.

Note that, the conditions assumed in those two regimes are somewhat different.
In our setting, loss functions are assumed to be upper smooth and satisfy some normalizing condition to be logarithmic losses,
while the boundedness and convexity of loss functions is often assumed in online learning.
Moreover, we have employed Bayesian predictors,
whereas more simple online predictors are typically used in the context of the online learning.

\subsection{Comparison to Minimax Risk over $\ell_1$-balls}
\label{sec:discussion_minimax_risk}
In the literature of high-dimensional statistics,
the minimax rate of \emph{statistical risk} is also
achieved with $\ell_1$-regularization~\citep{donoho1994minimax},
when the true parameter $\theta$ is in the unit $\ell_1$-ball.
Although both risk and regret are performance measures of prediction,
there are two notable difference.
One is that risks are calculated under some assumptions on true statistical distribution,
whereas regrets are defined without any assumptions on data.
The other is that risks are typically considered with in-model predictor,
i.e., predictors are restricted to a given model,
whereas regrets are often considered with out-model predictors
such as Bayesian predictors and online predictors.
Therefore, the minimax regret can be regarded as a more agnostic complexity measure
than the minimax risk.

If we assume Gaussian noise models and adopt the logarithmic loss functions,
the minimax rate of the risk is given as $\sqrt{2L\ln d/\sqrt{L}}$ according to \cite{donoho1994minimax}.
Interestingly,
this is same with the rate of the regret bound given by Theorem~\ref{thm:regret_bound} where $L=Ln$.
Moreover, the minimax-risk optimal penalty weights $\lambda$
is also minimax-regret optimal in this case.
Therefore, if the dimensionality $d$ is large enough compared to $L$ ($n$ in case of online-learning),
making no distributional assumption on data costs nothing
in terms of the minimax rate.

\section{Conclusion}
\label{sec:conclusion}

In this study,
we presented a novel characterization of the minimax regret for logarithmic loss functions, called the envelope complexity,
with $\ell_1$-regularization problems.
The virtue of the envelope complexity is that
it is much easier to evaluate than the minimax regret itself
and able to produce upper bounds systematically.
Then, using the envelope complexity,
we have proposed the spike-and-tails~(ST) prior, which almost achieves the luckiness minimax regret against smooth loss functions under $\ell_1$-penalization.
We also show that the ST prior actually adaptively achieves the 2-approximate minimax regret under high-dimensional asymptotics $\omega(1)=\ln d/\sqrt{n}=o(n)$.
In the experiment, we have confirmed our theoretical results:
The ST prior outperforms the tilted Jeffreys prior where the dimensionality $d$ is high, whereas the tilted Jeffreys prior is optimal
if $n\gg d$.

\paragraph{Limitation and future work}
The present work is relying on the assumption of the smoothness and logarithmic property on the loss functions.
The smoothness assumption may be removed by considering the smoothing effect of stochastic algorithms like stochastic gradient descent as in~\cite{kleinberg2018alternative}.
As for the logarithmic assumption,
it will be generalized to evaluate complexities with non-logarithmic loss functions with the help of tools that have been developed in the literature of information theory such as in~\cite{yamanishi1998decision}.
Finally, since our regret bound with the ST prior is quite simple~(there are only the smoothness $L$ and the radius $B$ except with the logarithmic term),
applying these results to concrete models such as deep learning models
would be interesting future work as well as the comparison to the existing generalization error bounds.


\appendix

\section{Asymptotic Lower Bound of Shtarkov Complexity for Standard Normal Location Models}
\label{sec:lower_bound_sc_std_norm_loc}
We show an asymptotic lower bound of the Shtarkov complexity of standard normal location models.
\begin{lemma}
    \label{lem:lower_bound_sc_std_norm_loc}
    Consider the $d$-dimensional standard normal location model,
    given by $f_X(\theta)=\frac12\norm{X-\theta}_2^2+\frac{d}2\ln 2\pi$,
    where $X\in\Xcal=\RR^d$.
    Let $\gamma=\lambda\norm{\theta}_1$ for $\lambda\ge 0$.
    Then we have
    \begin{align*}
        S(\gamma)\ge
        d\ln \rbr{1+\frac{e^{-\lambda^2/2}}{\sqrt{2\pi}\lambda^3}(1+o(1))}.
    \end{align*}
\end{lemma}
\begin{proof}
    By definition of $S(\gamma)$, we have
    \begin{align*}
        S(\gamma)&=\ln\int e^{-m(f_X+\gamma)}\nu(\d X)
        \\&=d\ln\int_{-\infty}^\infty \frac1{\sqrt{2\pi}}\sup_{t\in\RR}\exp
        \sbr{-\frac12 (x-t)^2-\lambda\abs{t} } \d x
        \\&=d\ln\frac1{\sqrt{2\pi}}\left[
        \int_{-\infty}^{-\lambda} e^{-\lambda(-\lambda-x)-\frac{\lambda^2}2} \d x
        +\right.
        \\&\qquad\qquad\qquad
        \left.
        \int_{-\lambda}^\lambda e^{-\frac{x^2}2} \d x
        +\int_\lambda^\infty e^{-\lambda(x-\lambda)-\frac{\lambda^2}2} \d x
        \right]
        \\&=d\ln\sbr{
        2\Phi(\lambda)-1+ 
        \frac{2e^{-\lambda^2/2}}{\sqrt{2\pi}}
        \int_0^\infty e^{-\lambda x}\d x
        }
        \\&=d\ln\sbr{
        2\Phi(\lambda)-1
        +\sqrt{\frac{2}{\pi}}\frac{e^{-\lambda^2/2}}{\lambda}
        },
    \end{align*}
    where $\Phi(\lambda)$ denotes the standard normal distribution function.
    Now, by~\cite{komatu1955elementary},
    $\Phi(\lambda)$ is bounded below with
    $\Phi(\lambda)>1-2\phi(\lambda)/(\sqrt{2+x^2}+x)$ for $\phi(\lambda)$ being the standard normal density,
    which yields the lower bound of interest after a few lines of elementary calculation.
\end{proof}

\section{Lower Bound on Minimax Regret of Smooth Models}
\label{sec:proof_minimaxity}

We describe how we adopt the minimax risk lower bound as
to show the minimax-regret lower bound.

The story of the proof is based on \cite{donoho1994minimax}.
First, the so-called three-point prior is constructed
to approximate the least favorable prior.
Then, since the approximate prior violates the $\ell_1$-constraint,
the degree of the violation is shown to be appropriately bounded
to derive a valid lower bound.

The goal of our proof is to establish
a lower bound on the minimax regret with respect to logarithmic losses,
whereas their proof is about the minimax risk with respect to $\ell_q$-loss.
Therefore, below we present the proof highlighting
(i) an approximate least favorable prior for \emph{logarithmic losses} over $\ell_1$-balls
and (ii) the way to bound \emph{regrets} on the basis of risk bounds.

Let $\Hcal=\mysetinline{\theta\in \RR^d}{\norm{\theta}_1\le B}$ be a $\ell_1$-ball.
Let $X\sim \Ncal_d[\theta, I_d/L]$ be a $
d$-dimensional normal random variable with mean $\theta\in H$ and precision $L>0$.
We denote the distribution just by $X\sim\theta$ where any confusion is unlikely.
Let $h\in\allh$ be a predictor associated with any sub-probability distribution $P(\cdot|h)\in\Mcal_+(\RR^d)$.
For notational simplicity, we may write $f_X(\theta)=\frac{L}2 \norm{X-\theta}_2^2+\frac{d}2 \ln \frac{2\pi}{L}$ and $f_X(h)=\ln \frac{\d P(X|h)}{\d \nu}$ where $\nu$ is the Lebesgue measure over $\RR^d$.

Consider the risk function
\begin{align*}
    R_d(h, \theta) &\eqdef \EE_{X\sim \theta} \sbr{f_X(h)-f_X(\theta)},
\end{align*}
and the Bayes risk function
\begin{align*}
    R_d(h, \pi) &\eqdef \EE_{\theta\sim \pi} \sbr{R_d(h,\theta)},
\end{align*}
where $\pi\in\Pcal(\Hcal)$ denotes prior distributions on $\Hcal$.
Then, the minimax Bayes risk bounds below the minimax regret,
\begin{align*}
    \reg^\star(\Hcal)
    &=\inf_{h\in\allh} \sup_{\theta\in\Hcal}\sup_{X\in \RR^d} f_X(h)-f_X(\theta)
  \\&\ge\inf_{h\in\allh} \sup_{\pi\in \Pcal(\Hcal)}
    \EE_{\theta\sim \pi}\EE_{X\sim \theta}\sbr{f_X(h)-f_X(\theta)}
    \\&= \inf_{h\in\allh} \sup_{\pi\in \Pcal(\Hcal)} R_d(h, \pi).
\end{align*}
The minimax theorem states that there exists a saddle point $(h^*, \pi_*)$ such that
\begin{align*}
    R_d(h^*,\pi_*)
    &=\inf_{h\in\allh} \sup_{\pi\in \Pcal(\Hcal)} R_d(h, \pi)
    \\&=\sup_{\pi\in \Pcal(\Hcal)} \inf_{h\in\allh} R_d(h, \pi)
    \eqdef \sup_{\pi\in\Pcal(\Hcal)}R_d(\pi),
\end{align*}
and $\pi_*$ is referred to as the least favorable prior.
We want to approximate $\pi_*$ to give an analytic approximation of $R_d(\pi_*)$, which is a lower bound of $\reg^\star(\Hcal)$.

Let $F_{\epsilon,\mu}\in\Pcal(\RR)$ be the three-point prior defined by
\begin{align*}
    F_{\epsilon,\mu}=(1-\epsilon)\delta_0+\frac\epsilon2\rbr{\delta_{-\mu}+\delta_\mu}
\end{align*}
for $\epsilon,\mu>0$.
We show that the corresponding achievable Bayes risk $R_1(F_{\epsilon,\mu})$ tends to be the entropy of the prior $F_{\epsilon,\mu}$ in some limit of small $\epsilon$.

\begin{lemma}
    \label{lem:three_point_risk}
    Take $\mu=\mu(\epsilon)=\sqrt{2L^{-1}\ln\epsilon^{-1}}$.
    Let $H_\epsilon=H(F_{\epsilon,\mu})=(1-\epsilon)\ln (1-\epsilon)^{-1}+\epsilon \ln 2\epsilon^{-1}$ be the entropy of the prior.
    Then we have
    \begin{align*}
        R_1(F_{\epsilon,\mu})\sim H_\epsilon \sim \epsilon \ln \frac1\epsilon
    \end{align*}
    as $\epsilon\to0$. Here, $x\sim y$ denotes the asymptotic equality
    such that $x/y\to 1$.
\end{lemma}
\begin{proof}
    First, we show the famous inequality on the entropy given by $R_1(F_{\epsilon,\mu})\le H_\epsilon$.
    Let $P(\cdot|h)=\EE_{\theta\sim F_{\epsilon,\mu}}P(\cdot |\theta)=
    (1-\epsilon)P(\cdot|0)+\frac{\epsilon}2\rbrinline{P(\cdot|-\mu)+P(\cdot|\mu)}$
    be the Bayes marginal distribution with respect to $F_{\epsilon,\mu}$.
    Then we have
    \begin{align*}
        &H_\epsilon - R_1(F_{\epsilon,\mu})
        \\&= H_\epsilon - R_1(h, F_{\epsilon,\mu})
      \\&= H_\epsilon - \EE_{\theta\sim F_{\epsilon,\mu}} \EE_{X\sim \theta} \ln \frac{\d P(X|\theta)}{\d P(X|h)}
      \\&=H_\epsilon - 
        (1-\epsilon)\EE_{P(X|0)} \ln \frac{\d P(X|0)}{\d P(X|h)}
        -\epsilon \EE_{P(X|\mu )} \ln \frac{\d P(X|\mu)}{\d P(X|h)}
      \\&=
        (1-\epsilon)\EE_{P(X|0)}
        \ln \rbr{1+\frac\epsilon{1-\epsilon}\frac{\d P(X|\mu)+\d P(X|-\mu)}{2\d P(X|0)}}
      \\&\quad +\epsilon \EE_{P(X|\mu)}
        \ln \rbr{1+\frac{1-\epsilon}\epsilon\frac{2\d P(X|0)+\d P(X|-\mu)}{\d P(X|\mu)}}
      \\&\ge 0.
    \end{align*}
    Now, we show that, with the specific value of $\mu=\mu(\epsilon)$,
    the gap is negligible compared to the entropy itself.
    Applying Jensen's inequality, we have
    \begin{align*}
        &H_\epsilon - R_1(F_{\epsilon,\mu})
        \\&\le 
        \epsilon
        +\epsilon \EE_{P(X|\mu)}
        \ln \rbr{1+\rbr{1-\epsilon}\rbr{2e^{-L\mu X}+\epsilon^3 e^{-2L\mu X}}}
      \\&\le 
        \epsilon (1 + \ln 4 + \EE_{P(X|\mu)} \max\cbr{0,\ -2L\mu X})
      \\&=
        \epsilon \rbr{1 + \ln 4 + \EE_{Z\sim \Ncal[0,1]} \max\cbr{0,\ 2\sqrt{L}\mu (Z-\sqrt{L}\mu)}}
        \\&\qquad(\because -\sqrt{L}(X-\mu)=Z)
      \\&\le
        \epsilon \rbr{1 + \ln 4 + 2\sqrt{L}\mu\epsilon}
        \\&=\epsilon \rbr{1 + \ln 4 + 2\epsilon\sqrt{2\ln\frac1\epsilon}}
        =o(H_\epsilon).
    \end{align*}
    Thus we get $H_\epsilon\sim R_1(F_{\epsilon,\mu})$.
\end{proof}

Now we show that the $d$-th Kronecker product of $F_{\epsilon,\mu}$,
$F^d_{\epsilon,\mu}$,
can be used to bound the Bayes minimax risk $R_d(\pi_*)$ with an appropriate choice of $\epsilon$ and $\mu$.
To this end, let $\pi_+=F^d_{\epsilon,\mu}\mid \Hcal$
be the conditional prior restricted over the $\ell_1$-ball $\Hcal$.
\begin{lemma}
    \label{lem:approx_lfp}
    Take $\epsilon\mu =(1-c)B/d$ and $\mu=\sqrt{2L^{-1} \ln \epsilon^{-1}}$ for $0<c<1$.
    Then, if $\epsilon\to 0$ and $d\epsilon\to \infty$,
    we have
    \begin{align*}
        R_d(\pi_*)\ge R_d(\pi_+) \sim R_d(F^d_{\epsilon,\mu})\sim d\epsilon \ln \frac1\epsilon.
    \end{align*}
\end{lemma}
\begin{proof}
    First of all, the inequality is trivial from the definition of $R_d(\pi)$.
    Moreover, the second asymptotic equality immediately follows from Lemma~\ref{lem:three_point_risk}.

    Now we consider the first asymptotic equality.
    Let $h$ be the Bayes minimax predictor with respect to the prior $F_{\epsilon,\mu}$ and $h^+$ be the one with respect to the conditional prior $\pi_+$.
    Then we have
    \begin{align*}
        &R_d(F^d_{\epsilon,\mu})
        \\&= R_d(h, F^d_{\epsilon,\mu})
      \\&= \EE_{\theta\sim F^d_{\epsilon,\mu}} \sbr{R_d(h, \theta)}
      \\&= F^d_{\epsilon,\mu}(\Hcal) R_d(h, \pi_+)
        + \EE_{\theta\sim F^d_{\epsilon,\mu}} \sbr{R_d(h, \theta)\cdot \1\cbr{\theta\notin \Hcal}}
        \\&\ge  F^d_{\epsilon,\mu}(\Hcal) \cdot R_d(\pi_+)
    \end{align*}
    and
    \begin{align*}
        &R_d(F^d_{\epsilon,\mu})
        \\&\le R_d(h^+, F^d_{\epsilon,\mu})
      \\&= \EE_{\theta\sim F^d_{\epsilon,\mu}} \sbr{R_d(h^+, \theta)}
      \\&= F^d_{\epsilon,\mu}(\Hcal) \cdot R_d(\pi_+)
        + \EE_{\theta\sim F^d_{\epsilon,\mu}} \sbr{R_d(h^+, \theta)\cdot \1\cbr{\theta\notin \Hcal}}.
    \end{align*}
    Let $N$ be the number of nonzero elements in $\theta\sim F^d_{\epsilon,\mu}$.
    Then $N$ is subjects to the Binomial distribution
    $\mathrm{Bin}(d, \epsilon)$.
    On the other hand, the event $\theta\in \Hcal$ is equal to
    $\cbr{\norm{\theta}_1\le B}=\cbr{N\le B/\mu=\EE N/(1-c)}$.
    Therefore, applying the Chebyshev's inequality, we get
    \begin{align*}
        P_d\eqdef F^d_{\epsilon,\mu}(\Hcal^c)
        &= \mathrm{Pr}\cbr{\frac{N - \EE N}{\EE N}> \frac{c}{1-c}}
        \le \frac{(1-c)^2}{c^2d \epsilon}\to 0.
    \end{align*}
    Similarly, we have $\EE \abs{N - \EE N}/\EE N\to 0$.
    Now observe that
    \begin{align*}
        &\EE_{\theta\sim F^d_{\epsilon,\mu}} \sbr{R_d(h^+, \theta)\cdot \1\cbr{\theta\notin \Hcal}}
        \\&\le \EE_{\theta\sim F^d_{\epsilon,\mu}} \EE_{\varphi\sim \pi_+} \sbr{R_d(\varphi, \theta)\cdot \1\cbr{\theta\notin \Hcal}}.
      \\&\le 2L \EE_{\theta\sim F^d_{\epsilon,\mu}} \EE_{\varphi\sim \pi_+} \sbr{\rbr{\norm{\varphi}_2^2+\norm{\theta}_2^2}\cdot \1\cbr{\theta\notin \Hcal}}
      \\&\le 2L\mu^2 \EE\sbr{
            P_d N + {N \cdot \1\cbr{N>B/\mu}}
        }
        \\& \qquad(\because \norm{\theta}_2^2=\mu^2 N)
      \\&\le 2L\mu^2 \EE N\rbr{
            2 P_d  + \frac{\EE \abs{N-\EE N}}{\EE N}
        }
        \\&= 4d\epsilon \ln \frac1\epsilon\rbr{
            2 P_d  + \frac{\EE \abs{N-\EE N}}{\EE N}
        }.
        \\&= o(R_d(F^d_{\epsilon,\mu})).
    \end{align*}
    Thus, combining all above, we get
    \begin{align*}
        &(1+o(1))R_d(\pi_+)
        \\&=(1-P_d) R_d(\pi_+)
        \\&\le R_d(F^d_{\epsilon,\mu})
        \\&\le  (1-P_d) \cdot R_d(\pi_+)
        + \EE_{\theta\sim F^d_{\epsilon,\mu}} \sbr{R_d(h^*, \theta)\cdot \1\cbr{\theta\notin \Hcal}}.
        \\&= (1-o(1)) R_d(\pi_*) 
        + o(R_d(F^d_{\epsilon,\mu})),
    \end{align*}
    which implies the desired asymptotic equality $R_d(F_{\epsilon,\mu})\sim R_d(\pi_+)$.
\end{proof}

Summing these up,
we have an asymptotic lower bound on the minimax regret which is the same as the upper bound given by the ST prior within a factor of two~(see Theorem~\ref{thm:regret_bound}).
This implies that both the regret of the ST prior and the Bayes risk of the prior $\pi_+$ are tight with respect to the minimax-regret rate except with a factor of two.

\begin{theorem}[Lower bound on minimax regret]
    Suppose that $\omega(1)=\ln (d/\sqrt{L})=o(L)$.
    Then we have
    \begin{align*}
        \reg^\star(\Hcal)\gtrsim \frac{B}{2}\sqrt{2L\ln \frac{d}{\sqrt{L}}},
    \end{align*}
    where $x\gtrsim y$ means that there exists $y'\sim y$ such that
    $x\ge y'$.
\end{theorem}
\begin{proof}
    The assumptions of Lemma~\ref{lem:approx_lfp} are
    satisfied for all $0< c < 1$ since
    \begin{align*}
        \epsilon &\lesssim \epsilon\sqrt{\ln \frac1\epsilon}
        =\frac{1-c}d\sqrt{\frac L2}\to 0,
        \\
        d\epsilon&=\rbr{1-c}\sqrt{\frac{L}{2\ln \frac1\epsilon}}
        \sim (1-c)\sqrt{\frac{L}{2\ln \frac{d}{\sqrt L}}}\to \infty.
    \end{align*}
    Thus, we have 
    \begin{align*}
        \reg^\star(\Hcal)\ge R_d(\pi_*)
        \gtrsim d\epsilon \ln \frac1\epsilon
        \sim \rbr{1-c}\frac{B}2\sqrt{2L\ln \frac{d}{\sqrt{L}}}
    \end{align*}
    for all $0<c<1$.
    Slowly moving $c$ toward zero completes the theorem.
\end{proof}

\section{Existence of Gap between $\lreg^\star$ and $\lreg^\bayes$ under $\ell_1$-Penalty}
\label{sec:gap_bayesian_lreg}
Below we show that,
under standard normal location models,
the Bayesian luckiness minimax regret is strictly larger than the non-Bayesian luckiness minimax regret
if $\gamma$ is nontrivial and has a non-differentiable point.
Here we refer to $\gamma$ as \emph{trivial} when
there exists $\theta_0$ such that $\gamma(\theta)=\infty$ for all $\theta\neq \theta_0$.

\begin{lemma}
    \label{lem:gap_bayesian_lreg}
    Let $f_X(\theta)=\frac12\rbr{X-\theta}^2+\frac12\ln 2\pi$ for $X\in\RR$ and $\theta\in\RR$.
    Then, for all nontrivial, convex and non-differentiable penalties $\gamma:\RR\to \RRbar$, 
    \begin{align*}
        \lreg^\star(\gamma)<\lreg^\bayes(\gamma).
    \end{align*}
\end{lemma}
\begin{proof}
    Let $\Fcal=\mysetinline{f_X}{X\in\RR}$ and 
    recall that $\lreg^\bayes(\gamma)=\inf_{w\in\Ecal(\Fcal_\gamma)}\ln w\sbr{e^{-\gamma}}$ by Theorem~\ref{thm:set_repr_envelope}.
    Let $\norm{\cdot}_\gamma$ be the metric of pre-priors
    $w\in\Mcal_+(\RR)$ given by $\norm{w}_\gamma=w\sbr{e^{-\gamma}}$.
    Owing to the continuity of $w\mapsto \ln w\sbr{e^{-\gamma}}$ and the completeness of $\Ecal(\Fcal_\gamma)\subset \Mcal_+(\RR)$,
    it suffices to show that there exists no pre-prior $w\in \Ecal(\Fcal_\gamma)$
    such that $\ln w\sbr{e^{-\gamma}}=S(\gamma)$.
    Let us prove this by contradiction.
    Now, assume that $\ln w\sbr{e^{-\gamma}}=S(\gamma)$.
    Observe that
    \begin{align*}
        0
        &=w\sbr{e^{-\gamma}}-\exp S(\gamma)\nonumber
        \\&=w\sbr{\int e^{-f_X-\gamma}\nu(\d X)}-\int e^{-m(f_X+\gamma)}\nu(\d X)
        \\&=\int \cbr{w\sbr{e^{-f_X-\gamma}}-e^{-m(f_X+\gamma)}}\nu(\d X),
    \end{align*}
    which means $w\sbr{e^{-f_X-\gamma}}=e^{-m(f_X+\gamma)}$ for almost every $X$
    since $w\in\Ecal(\Fcal_\gamma)$.
    Note that $f_X(\theta)$ is continuous with respect to $X$,
    and then we have $w\sbr{e^{-f_X-\gamma}}=e^{-m(f_X+\gamma)}$ for all $X$.
    After some rearrangement and differentiation, we have
    \begin{align}
        0
        &=\frac\d{\d X}w\sbr{e^{-f_X-\gamma+m(f_X+\gamma)}}
        \nonumber
        \\&=w\sbr{\frac{\d e^{-f_X-\gamma+m(f_X+\gamma)}}{\d X}}
        \nonumber
        \\&=w_{\theta}\sbr{\rbr{\theta-\theta^*_X} e^{-f_X-\gamma+m(f_X+\gamma)}},
        \label{eq:lem_tmp_34789723848902457487}
    \end{align}
    where $\theta^*_X=\arg m(f_X+\gamma)$.
    Here we exploited Danskin's theorem at the last equality.
    One more differentiation gives us
    \begin{align*}
        0
        &=\frac{\d}{\d X}w_\theta\sbr{\rbr{\theta-\theta^*_X} e^{-f_X-\gamma+m(f_X+\gamma)}},
        \\&=w_\theta\sbr{\cbr{\rbr{\theta-\theta^*_X}^2- \frac{\d\theta^*_X}{\d X}} e^{-f_X-\gamma+m(f_X+\gamma)}}
    \end{align*}
    for all $X\in \RR$.
    
    Note that we have $\frac{\d\theta^*_X}{\d X}|_{X=t}=0$
    for any non-differentiable points $t$ of $\gamma$.
    Then it implies that $w=c\delta_{\theta_t^*}$ where $\delta_s$ denotes the Kronecker delta measure.
    Then, according to \eqref{eq:lem_tmp_34789723848902457487}, we have
    \begin{align*}
        0
        &=w_\theta\sbr{\rbr{\theta-\theta^*_X} e^{-f_X-\gamma+m(f_X+\gamma)}}.
        \\&=c\rbr{\theta^*_t-\theta^*_X} e^{-f_X(\theta^*_t)-\gamma(\theta^*_t)+m(f_X+\gamma)},
    \end{align*}
    which means that $\theta^*_X=\theta^*_t$ is a constant independent of $X$.
    However, this contradicts to the assumption
    that $\gamma$ is nontrivial.
\end{proof}

As a remark, we note that this lemma is easily extended to multidimensional exponential family of distributions.

\end{document}